\documentclass[twoside,12pt]{article}
\usepackage{graphicx} 
\usepackage[T1]{fontenc}    
\usepackage{url}            
\usepackage[usenames,dvipsnames,svgnames,table]{xcolor}
\usepackage{amsmath,tikz}
\usepackage{subfigure}
\usetikzlibrary{patterns,matrix,arrows,positioning}
\usepackage{graphicx} 
\usepackage{tikz}
\usetikzlibrary{matrix}
\usepackage[colorlinks=true, allcolors=blue]{hyperref}
\usepackage{xspace}
\usepackage{easyReview}
\usepackage{setspace}
\usepackage{rotating,tabularx}
\usepackage{scrextend}
\usepackage[margin=1.9cm]{geometry}
\usepackage[most]{tcolorbox}
\usepackage{amsfonts}
\usepackage{authblk}
\usepackage{soul}
\usepackage{amsmath, amsthm}
\usepackage{enumitem}
\usetikzlibrary{shapes.geometric}

\newtheorem{theorem}{Theorem}
\newtheorem{lemma}[theorem]{Lemma}
\newtheorem{corollary}[theorem]{Corollary}

\newtcolorbox{mytextbox}[1][]{%
  sharp corners,
  enhanced,
  colback=white,
  attach title to upper,
  #1
}

\title{Beyond Words: A Mathematical Framework for Interpreting Large Language Models}
\author{Javier González}
\author{Aditya V. Nori}
\affil{Microsoft Research}
\date{}

\begin{document}

\maketitle


\begin{abstract}
Large language models (LLMs) are powerful AI tools that can generate and comprehend natural language text and other complex information. However, the field lacks a \textit{mathematical framework} to systematically describe, compare and improve LLMs. We propose \textsc{Hex} a framework that clarifies key terms and concepts in LLM research, such as hallucinations, alignment, self-verification and chain-of-thought reasoning. The \textsc{Hex} framework offers a precise and consistent way to characterize LLMs, identify their strengths and weaknesses, and integrate new findings. Using \textsc{Hex}, we differentiate chain-of-thought reasoning from chain-of-thought prompting and establish the conditions under which they are equivalent. This distinction clarifies the basic assumptions behind chain-of-thought prompting and its implications for methods that use it, such as self-verification and prompt programming.

Our goal is to provide a formal framework for LLMs that can help both researchers and practitioners explore new possibilities for generative AI. We do not claim to have a definitive solution, but rather a tool for opening up new research avenues. We argue that our formal definitions and results are crucial for advancing the discussion on how to build generative AI systems that are safe, reliable, fair and robust, especially in domains like healthcare and software engineering.
\end{abstract}


\section{Introduction}\label{sec:introduction}

Large language models (LLMs) are powerful tools that can generate text responses to various input prompts, such as queries or commands. To interact with an LLM, a human user provides a prompt and receives a response from the model, which is based on its internal algorithms. The user then evaluates the response, and decides whether it answers the prompt satisfactorily. This process may be repeated several times until the user reaches a desired outcome (or gives up).

State-of-the-art LLMs, such as PaLM \cite{Chowdhery2022PaLMSL}, GPT-4 \cite{OpenAI2023GPT4TR} or LLaMA \cite{meta}, can exhibit both remarkable and baffling behaviours, depending on the task. Peter Lee captures this paradox by saying that they are both "\textit{smarter and dumber than you}" at math, statistics and logic~\cite{schubert2023gpt4}. What causes this discrepancy? Research has primarily focused on testing LLMs on different tasks so the user can gain insights about their behaviour. For GPT-4, a large collection of behaviours have been studied on several tasks \cite{Bubeck2023SparksOA}. Benchmarking efforts have been conducted in various domains such as causal discovery \cite{DBLP:journals/corr/abs-2305-00050}, summarization \cite{Zhang2023BenchmarkingLL} and reasoning \cite{sawada2023arb}. These research directions are reasonable, given that most LLMs are black-boxes that we cannot easily inspect, interpret or explain. 


In this paper, we offer a new perspective on interpreting LLMs and present a mathematical framework that formalises and generalises what is already known about LLMs. To advance LLM research, we pursue two objectives: First, we define key LLM concepts such as alignment \cite{shen2023large}, hallucinations \cite{Ji_2023} and chain-of-thought  reasoning \cite{wei2023chainofthought} using our mathematical framework. Second, we  provide a reasoning tool that facilitates a common understanding of existing LLM research and form a basis for exploring new questions and challenges.  Our framework, called \textsc{Hex}, is based on commutative diagrams that relate different levels of abstraction and computation in LLMs. The framework itself does not aim to provide all the answers but it represents a framework that can be used by researchers in LLMs to share, position and reason about their findings. 

We show how \textsc{Hex} can capture the essence of LLMs as abstract execution machines that use natural language as an interface, and how it can reveal the assumptions and implications of various methods that use LLMs, such as self-verification and prompt programming. We also show how \textsc{Hex} can formalize and evaluate the output of LLMs using distance metrics and compatibility relations. 

Our goal is to provide a formal framework for LLMs that can help both researchers and practitioners explore new possibilities for generative AI. We do not claim to have a definitive solution, but rather a tool for opening up new research avenues. We argue that our formal definitions and results are crucial for advancing the discussion on how to build generative AI systems that are safe, reliable, fair and robust, especially in domains like healthcare and  software engineering.


The rest of the paper is organized as follows. Section~\ref{sec:math-framework} introduces the \textsc{Hex} framework and its basic definitions. Section~\ref{sec:concepts_definitions} shows how the \textsc{Hex} framework can formalize common LLM concepts, such as hallucination, alignment, and chain-of-thought reasoning. Section~\ref{sec:cot} discusses the conditions of when chain-of-thought prompting elicits chain-of-thought reasoning, and their broader implications for LLM methods and how they are used in practice. Section~\ref{sec:discussion} concludes the paper and outlines future work.

\section{The mathematical framework}
\label{sec:math-framework}

LLMs are abstract execution machines that use natural language as an interface. Given an input prompt that captures the user's intention, an LLM produces an output answer in 
the form of text that is presented to the user. The prompt aims to solve a problem, such as a mathematical operation, a summarization, or the generation of a certain type of text. The  prompt always expresses an  \emph{intention} to execute some operations based on some provided information. The LLM, which operates as a black-box, uses the internal memory acquired at training time to provide an answer by sampling from some probability distribution conditioned on the input prompt\footnote{This is computed by choosing the next word that best fits the previous text, and repeating the process for each word.}.

This process has become extremely powerful in the latest generation of LLMs, creating a new set of terms for AI research and practice. Alignment, hallucinations, chain-of-thought reasoning, self-verification and prompt programming are some of the concepts that AI researchers and practitioners use routinely -- our work allows us to formally define these concepts using a unified mathematical framework (Section \ref{sec:concepts_definitions}). We introduce the  definitions  needed for our framework next.

\begin{description}
\item [Concrete state:] A concrete state is a function that assigns values from a concrete domain to a (finite) set of variables. We use $\mathcal{V}$ to denote the set of variables, and $\mathcal{C}$ to denote the concrete domain. We write $\sigma: \mathcal{V} \rightarrow \mathcal{C}$ to indicate that $\sigma$ is a concrete state. For example, suppose $\mathcal{V} = \{x, y\}$ and $\mathcal{C} = \mathbb{Z}$, the set of integers. Then, one possible concrete state is $\sigma = \{ x\mapsto 1, y\mapsto 2 \}$, which means that $\sigma(x) = 1$ and $\sigma(y) = 2$. The set of all concrete states over $\mathcal{V}$ and $\mathcal{C}$ is denoted by $2^{\mathcal{V} \rightarrow \mathcal{C}}$. 

\item[Problem:] Let $q : 2^{\mathcal{V} \rightarrow \mathcal{C}} \rightarrow 2^{\mathcal{V} \rightarrow \mathcal{C}}$ be a query, and let $\sigma$ be a concrete state. Informally, the query $q$ expresses what a user wishes to compute over a state $\sigma$. For example, $q$ might select some variables, apply some operations, or filter some conditions.  A problem is defined by the query-state pair $(q,\sigma)$.

\item [Prompt:] A prompt is a natural language expression of a problem $(q, \sigma)$ that an LLM can solve\footnote{Both elements can be mixed in some prompts. Although our theory is general, we separate them to simplify our exposition.}. A prompt $(\overline{q}, \overline{\sigma})$ consists of two strings: a query string $\overline{q}$ that represents the query $q$ in words, and a state string $\overline{\sigma}$ that describes the concrete state $\sigma$ in words. We denote by $\mathcal{T}$ the set of all strings.

\item [LLM:] Given a prompt $(\overline{q}, \overline{\sigma})$ as input, an LLM $M$ interprets this to produce a natural language or string output. 

\item [Abstract state:]   A state string $\overline{\sigma}$ has a meaning that can be captured by a vector, which is an abstract state $\hat{\sigma}$  or an embedding. An LLM works with these vectors to manipulate state strings. We denote by $\mathcal{A}$ the set of all abstract states.

\item [Abstraction map:] An abstraction map $\alpha$ is a function that transforms a concrete state $\sigma$ into an abstract state $\hat{\sigma}$. The function $\alpha$ consists of two sub-functions: (i) $\alpha_c$, which converts the concrete state $\sigma$ to a state string $\bar{\sigma}$ (this would be part of the prompt design process) and (ii) $\alpha_a$, which encodes the text state $\bar{\sigma}$ into an abstract state $\hat{\sigma}$  (this is the  LLM encoder). In other words, $\alpha(\sigma) = (\alpha_a \circ \alpha_c) (\sigma)$.

\item [Concretization map:] a concretization map $\gamma$ concretizes an abstract state $\hat{\sigma}$ by producing a concrete state $\sigma$. The function $\gamma$ consists of two steps: (i) $\gamma_a$, which decodes the abstract state $\hat{\sigma}$ into a state string $\bar{\sigma}$ using an LLM (the LLM decoder), and (ii) $\gamma_c$, which interprets the state string $\bar{\sigma}$ as a description of the concrete output state (the human interpretation of the LLM output). In other words, $\gamma(\hat{\sigma}) = (\gamma_c \circ \gamma_a) (\hat{\sigma})$.
\end{description}

Figure~\ref{fig:hex} shows the \textsc{Hex} diagram\footnote{A commutative diagram with six edges.} for an LLM $M$ and a problem $(q,\sigma)$, where $q$ is a query and $\sigma$ is a concrete state. The corners of the diagram are states and the vertical and horizontal edges connect different levels of abstraction. The vertical edges are based on the abstraction maps $\alpha_c$ and $\alpha_a$, and the concretization maps $\gamma_a$ and $\gamma_c$. The horizontal edges are based on the functionals $\Lambda_q$ and $\Lambda_{\hat{q}}$, which compute the query $q$ or its abstract counterpart $\hat{q}$ on the corresponding states. These functionals are defined as follows:
\begin{description}
   \item $\Lambda_q(\sigma) \stackrel{\mathtt{def}}{=} q(\sigma)$ applies the query $q$ to the concrete state $\sigma$.
   \item $\Lambda_{\hat{q}}(\hat{\sigma}) \stackrel{\mathtt{def}}{=} \hat{q}(\hat{\sigma})$ computes the abstract query $\hat{q}$ (this is the LLM M's interpretation of the 
   query string $\overline{q}$) on the abstract state $\hat{\sigma} = \alpha(\sigma)$.
\end{description}

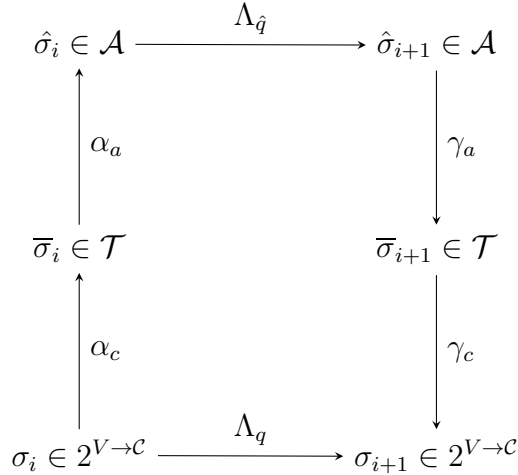
\begin{figure}[t!]
\centering
\begin{tikzpicture}
  \matrix (m) [matrix of math nodes,row sep=5em,column sep=6em,minimum width=2em]
  {  \hat{\sigma_i} \in \mathcal{A} & \hat{\sigma}_{i+1} \in \mathcal{A} \\
     \overline{\sigma}_i \in  \mathcal{T}&  \overline{\sigma}_{i+1} \in  \mathcal{T} \\
     \sigma_i \in 2^{V \rightarrow \mathcal{C}} & \sigma_{i+1}\in 2^{V \rightarrow \mathcal{C}} \\};
  \path[-stealth]
    (m-3-1) edge node [right] {$\alpha_c$} (m-2-1)
    (m-2-1) edge node [right] {$\alpha_a$} (m-1-1)
    (m-1-2) edge node [right] {$\gamma_a$} (m-2-2)
    (m-2-2) edge node [right] {$\gamma_c$} (m-3-2)
    (m-3-1) edge node [above] {$\Lambda_{q}$} (m-3-2)
    (m-1-1) edge node [above] {$\Lambda_{\hat{q}}$} (m-1-2);
\end{tikzpicture}
\caption{The \textsc{Hex} diagram for a problem $(q, \sigma_i)$.}
\label{fig:hex}
\end{figure}

A \textsc{Hex} diagram represents different ways or paths to solve a problem $(q, \sigma_i)$. Each path starts from an initial concrete state $\sigma_i$ and ends at a final concrete state $\sigma_{i+1}$. The diagram is \textit{commutative} if and only if every path leads to a \textit{compatible} $\sigma_{i+1}$. 
In other words, if and only if  
$\Lambda_q(\sigma_i) \equiv (\gamma_c \circ \gamma_a \circ \Lambda_{\hat{q}} \circ \alpha_a \circ \alpha_c) (\sigma_i)$
, where $\equiv$ is a suitable compatibility relation. Section~\ref{subsec:hallucination-and-misalignment} gives some examples of compatibility relations, such as equality. 

The path starting with the state string $\overline{\sigma}_i$ and ending in the state string $\overline{\sigma}_{i+1}$ traces  the execution of the LLM $M$.
The abstract states $\hat{\sigma}_i$ and $\hat{\sigma}_{i+1}$ and the function $\Lambda_{\hat{q}}$ are unobservable when the LLM is a blackbox.

\subsection{\textsc{Hex} examples}
\label{subsec:examples}

In this section, we show how the \textsc{Hex} elements work together via two examples. One example is an arithmetic problem that needs logical thinking, and the other is a code generation problem that needs both logic and some creativity.

\subsubsection{Simple arithmetic problem}\label{sec:example_arithmetic}

\textbf{Problem}: Solve for $z = x+y$ where $x =12$, $y=13$.

The \textsc{Hex} instance for this problem is defined as follows.
\begin{itemize}
    \item Concrete state: The set of variables $\mathcal{V} = \{x, y, z\}$, where $x$, $y$ and $z$ are the variables of the problem.  The concrete domain $\mathcal{C}$ is the set of integers $\mathbb{Z}$. The concrete initial state $\sigma_i = \{x \mapsto 12, y \mapsto 13, z \mapsto \perp, \}$, where $\perp$ represents an undefined value.
    \item Problem: The query $q = \lambda \sigma\ .\ \sigma(x)+\sigma(y)$, and the problem is the pair $(q, \sigma_i)$.
    \item The functional $\Lambda_q(\sigma_i)$ computes the concrete output state $\sigma_{i+1}$.
    \item $\sigma_{i+1} = \{x\mapsto 12, y\mapsto 13, z\mapsto 25 \}$. This is the correct answer that the user expects from the LLM, in other words, the ground truth.
    \item Abstraction map: The function $\alpha_c$ maps the concrete input state $\sigma_i$ to a state string $\overline{\sigma}_i$ which is encoded within the text prompt (shown below). The LLM then interprets the state string $\overline{\sigma}_i$ (represented by the substring \texttt{"x = 12, y = 13"}) and produces  $\hat{\sigma}_i = \alpha_c(\overline{\sigma}_i)$ which corresponds to an unobservable internal LLM state. 
    \item The functional $\Lambda_{\hat{q}}$ is an unknown, functional that computes the abstract query $\hat{q}$, specified by the query string $\overline{q} = $ \texttt{"Solve for z = x+y"} of the prompt shown below.
    \end{itemize}

\begin{mytextbox}
\textbf{Simple arithmetic problem} \\

\textbf{Prompt}: 

\begin{verbatim}
final_answer = GPT4("Solve for z = x+y, where x = 12, y = 13.")
\end{verbatim}

\textbf{Answer:}
\begin{verbatim}
The value of z is 25.
\end{verbatim}
\end{mytextbox}
   
\begin{itemize}
    \item The LLM output for \textbf{Prompt} is "\texttt{The value of z is 25.}", and this is represented by the state string $\overline{\sigma}_{i+1}$. 
    \item Concretization map: The concretization function $\gamma_c$ interprets the text state $\overline{\sigma}_{i+1}$ as a concrete state $\{x\mapsto 12, y\mapsto 13, z\mapsto 25 \}$ which is equal to the ground truth concrete state $\sigma_{i+1}$. This relies on the LLM computing the correct answer, which means that the corresponding \textsc{Hex} diagram is commutative (and consequently, there is alignment).
     
\end{itemize}

This example has a unique and precise solution, so we can easily check that the \textsc{Hex} diagram commutes. However, as we will see later, this is not always the case.

\subsubsection{Sketching a unicorn}

\textbf{Problem}: Write TikZ code to draw a unicorn. 

The \textsc{Hex} instance for this problem is defined as follows.
\begin{itemize}
    \item Concrete state: We use a set of variables $\mathcal{V} = \{x, body, legs, head, tail, horn, m\}$ to represent the information needed to draw an animal using TikZ, which is a language for creating graphics in LaTeX. The variable $x$ holds the name of the animal we want to draw (string value), such as unicorn, and the variables \textit{body}, \textit{legs}, \textit{head}, \textit{tail}, \textit{horn} hold the TikZ code for drawing each part of the animal, such as a circle, a line, or a curve. The variable $m$ is a slack variable that can be used to store any additional information or code that is not captured by the other variables. We start with an initial state $\sigma_i$, where $x$ is assigned the value \texttt{"unicorn"} and all the other variables are assigned the undefined value $\perp$, meaning that we do not have any code for drawing the unicorn yet.

    \item Problem: The query $q$ is a TikZ code generator that takes $\sigma_i$ as input draws a diagram of $\sigma_i(x)$, and the problem is the pair $(q, \sigma_i)$.
    
    \item The functional $\Lambda_q(\sigma_i)$ produces the concrete output state $\sigma_{i+1}$. In particular, $\sigma_{i+1} = \{x \mapsto \texttt{"unicorn"},\ body \mapsto  \texttt{code for body},\ legs\mapsto \texttt{code for legs},\ head \mapsto \texttt{code for head},\\ tail\mapsto \texttt{code for tail},\ horn\mapsto \texttt{code for horn},\ m \mapsto\perp \}$. This is the expected answer for the user, or the ground truth, which we define as any code that sketches the unicorn. However, there are many possible ways to write such code. In Section \ref{sec:llm_evaluation}, we explain how we measure the quality of the LLM's answers.
    
    \item Abstraction map: We use the function $\alpha_c$ to turn the concrete input state $\sigma_i$ into a state string $\overline{\sigma}_i$ that we can include in the text prompt (shown below). The state string is $\overline{\sigma}_i = $\texttt{"unicorn"} that the LLM has to interpret. The LLM then produces an internal state $\hat{\sigma}_i = \alpha_c(\overline{\sigma}_i)$ that we cannot observe. 
    
    \item We also have an unknown functional $\Lambda_{\hat{q}}$ that takes the string $\overline{q} =$ \texttt{"draw in Tikz"} from the prompt and computes the abstract query $\hat{q}$.
    \end{itemize}

\begin{mytextbox}
\textbf{Prompt}: 
\begin{verbatim}
answer = GPT4("You are a programmer that uses TikZ to generate diagrams: 
 draw a unicorn.")
\end{verbatim}

\textbf{Answer}: 
\begin{verbatim}
\begin{tikzpicture}
% Body
\draw[fill=gray!50] (0, 0) circle (1cm);
\draw[fill=gray!50] (1.5, 0) circle (1cm);
\draw (1, 0) ellipse (0.7cm and 1.5cm);
% Legs
\draw (0.4, -1) -- (0.2, -2.5);
\draw (0.6, -1) -- (0.8, -2.5);
\draw (1.9, -1) -- (1.7, -2.5);
\draw (2.1, -1) -- (2.3, -2.5);
% Head
\draw[fill=gray!50] (2.7, 1.4) circle (0.7cm);
\draw (2.3, 1.5) -- (2.1, 1.2);
% Tail
\draw (0.0, 0.7) to[out=120, in=30] (-1, 1.2);
% Horn
\draw[thick] (2.7, 2.05) -- (2.9, 3);
\end{tikzpicture}
\end{verbatim}
\textbf{Result after compiling the generated code:}
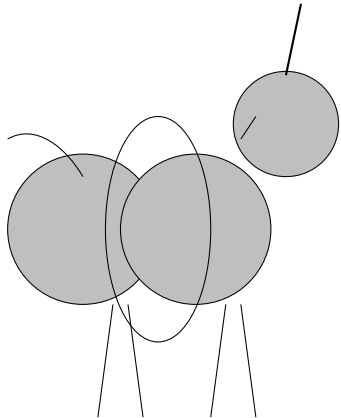
\begin{center}
\begin{tikzpicture}
\draw[fill=gray!50] (0, 0) circle (1cm);
\draw[fill=gray!50] (1.5, 0) circle (1cm);
\draw (1, 0) ellipse (0.7cm and 1.5cm);
\draw (0.4, -1) -- (0.2, -2.5);
\draw (0.6, -1) -- (0.8, -2.5);
\draw (1.9, -1) -- (1.7, -2.5);
\draw (2.1, -1) -- (2.3, -2.5);
\draw[fill=gray!50] (2.7, 1.4) circle (0.7cm);
\draw (2.3, 1.5) -- (2.1, 1.2);
\draw (0.0, 0.7) to[out=120, in=30] (-1, 1.2);
\draw[thick] (2.7, 2.05) -- (2.9, 3);
\end{tikzpicture}
\end{center}
\end{mytextbox}

\begin{itemize}
    \item The LLM output for \textbf{Prompt} is "\texttt{begin\{tikzpicture\}... end\{tikzpicture\}}", and this is the string state $\overline{\sigma}_{i+1}$.
    \item Concretization map: The concretization function $\gamma_c$ interprets the text state $\overline{\sigma}_{i+1}$ as a concrete state. In this example, the text state is a piece of code, so we can compile it to see what it does. The compiled code indeed draws a shape that resembles an unicorn.
\end{itemize}

Unlike the arithmetic example, this problem has multiple solutions and the relation that makes the \textsc{Hex} diagram commute depends on the scenario. For example, different human evaluations may have different standards for how much the figure resembles an unicorn. Others may care more about the clarity of the code.

\section{LLM concepts with \textsc{Hex}}\label{sec:concepts_definitions}

This section shows how the \textsc{Hex} framework can formalize common LLM concepts.

\subsection{Hallucination and misalignment}
\label{subsec:hallucination-and-misalignment}

The LLM literature often uses the terms \textit{hallucination} and \textit{misalignment} to describe two common issues with these models. Hallucination means that LLMs produce wrong or irrelevant responses to prompts, but make them sound reasonable or coherent \cite{Ji_2023}. This can be dangerous when LLMs are used in critical domains where accuracy and safety are important \cite{mckenna2023sources}. Alignment means that LLMs act in accordance with their human users' intentions \cite{shen2023large}. LLMs that are misaligned act differently from what their users want. This can also cause harm, such as giving wrong answers, generating biased outputs or discriminating results \cite{Wolf2023}. Alignment involves tuning LLMs to encourage desired behaviors and discourage undesired ones.

We use the \textsc{Hex} diagram to discuss the terms hallucination and misalignment, which are generic terms in the literature. Given a problem $(q, \sigma)$, we also need a compatibility relation $\equiv$ to compare the equivalence of two concrete states. Hallucination and misalignment occur when the \textsc{Hex} diagram for $(q, \sigma)$ does not commute, meaning that the paths from $\sigma$ to the output concrete state do not produce equivalent states under $\equiv$. Formally, this means that the compatibility relationship $(\gamma \circ \Lambda_{\hat{q}} \circ \alpha)(\sigma) \equiv \Lambda_q(\sigma)$,
where $(\gamma \circ \Lambda_{\hat{q}} \circ \alpha)(\sigma)$ is the answer generated by the LLM, does not hold.

We want to detect hallucinations or misalignments by comparing states, so we need a suitable notion of equivalence ($\equiv$) that captures the kind of errors we are interested in. For example, consider the simple arithmetic problem in Section~\ref{sec:example_arithmetic}, where the correct answer is $\{ x\mapsto 12, y\mapsto 13, z\mapsto 25, w \mapsto \perp \}$. If the LLM outputs "\texttt{The sum of x and y is 29}", this corresponds to the state $\{ x\mapsto 12, y\mapsto 13, z\mapsto 29, w \mapsto \perp \}$ under the concretisation function $\gamma$. This state is a hallucination, because it does not match the correct answer. 
One possible way to define an equivalence relation is to compare the values of variables in two concrete states $\sigma_1$ and 
$\sigma_2$. We say that $\sigma_1$ and $\sigma_2$ are equivalent, if they have the same value for every variable in 
$\mathcal{V}$. Formally, $(\sigma_1 \equiv \sigma_2) \iff \forall x \in \mathcal{V}\ .\ \sigma_1(x) = \sigma_2(x)$. According to this definition, the only state equivalent to $\{ x\mapsto 12, y\mapsto 13, z\mapsto 25, w \mapsto \perp \}$ is itself.

However, suppose the LLM outputs "\texttt{The sum of x and y is a tiger}". This answer is wrong and also misaligned, because it contains a nonsensical output that is not an integer. In this case, the only way for the two states to be aligned is if every variable in the scope has an integer value. That is, $(\sigma_1 \equiv \sigma_2) \iff \forall x \in \mathcal{V}\ .\ \mathit{typeof}(\sigma_1(x)) = \mathit{typeof}(\sigma_2(x)) = \mathtt{int}$.

To evaluate how well an LLM performs a task, we need to measure the extent of its hallucinations or misalignments. How can we do that? We discuss this in the next section.

\subsection{Evaluating the output of an LLM}\label{sec:llm_evaluation}

Given a problem $(q, \sigma)$, the output of an LLM $M$ is defined by the state string $\overline{\sigma} = (\gamma_a \circ \Lambda_{\hat{q}} \circ \alpha)(\sigma)$. Applying the  $\gamma_c$ function to $\overline{\sigma}$ results in the concrete state $\gamma_c(\overline{\sigma})$. The compatibility relationship introduced in Section~\ref{subsec:hallucination-and-misalignment} is crucial for determining whether the \textsc{Hex} diagram for $(q, \sigma)$ is commutative. If it is  not, a natural question to ask is how much that LLM's answer deviates from the one that would make the \textsc{Hex} diagram commutative.

We can evaluate how well an LLM solves a problem by comparing its output state, $\overline{\sigma}$, with a reference state that represents the correct answer, $\Lambda_q(\sigma)$. To do this, we need a distance metric $\Delta:2^{\mathcal{V} \rightarrow \mathcal{C}} \times 2^{\mathcal{V} \rightarrow \mathcal{C}} \rightarrow \mathbb{R}^{+}$ that quantifies how different two states are. The choice of $\Delta$ depends on the problem scenario, but it must satisfy a consistency condition: the \textsc{Hex} diagram must commute, meaning that the concrete state derived from the LLM's output, $\gamma_c(\overline{\sigma})$, must have zero distance from the reference ground truth state, $\Lambda_q(\sigma)$.

For example, suppose we are solving the problem of adding 12 + 13. If the LLM's output is $\overline{\sigma}$ and we apply $\gamma_c$ to it, we get $\gamma_c(\overline{\sigma}) = (12, 13, 19)^T$ (this is the vectorized state). This is the concrete state that represents the LLM's answer. However, the correct answer is 25, which we can get by applying $\Lambda_q$ to the input $\sigma$. This gives us the reference ground truth state $\Lambda_q(\sigma) = (12, 13, 25)^T$. To measure how far the LLM's answer is from the correct answer, if we use the L1-norm as the distance function, we get $\Delta(\gamma_c(\overline{\sigma}), \Lambda_q(\sigma)) = |19-25| = 6$. This means that the LLM's answer is off by 6 units. The \textsc{Hex} diagram does not commute in this case, because $\gamma_c(\overline{\sigma}) \neq \Lambda_q(\sigma)$.

Some problems need complex distance methods to measure how close the diagram is to being commutative. This is hard when there are many valid solutions for the ground truth $\Lambda_q(\sigma)$, such as in code generation where different solutions are acceptable. The user of the LLM needs to define these metrics to capture the right properties for these problems.

\section{Chain-of-Thought}
\label{sec:cot}

Chain-of-Thought (CoT) \cite{wei2023chainofthought} is a method for improving the ability of LLMs to perform various reasoning tasks, such as answering questions, solving math problems, or generating explanations. The idea is that some problems are too complex or abstract for the LLM to solve directly, and it can benefit from breaking them down into simpler or more concrete sub-problems. To do this, the method involves prompting the LLM to generate intermediate steps that lead to the final answer, and using these steps as additional context or feedback for the next step. 

For example, imagine drawing a hexagon with a circle inside it, and a square inside the circle. The steps 
(\textit{CoT reasoning}) are:
  \begin{enumerate}
    \item Draw a hexagon.
    \item Draw the circle inside the hexagon.
    \item Draw a square inside the circle.
  \end{enumerate}

The CoT reasoning method can help the model to avoid errors, gaps, or inconsistencies in its reasoning, and to produce more coherent and transparent outputs.

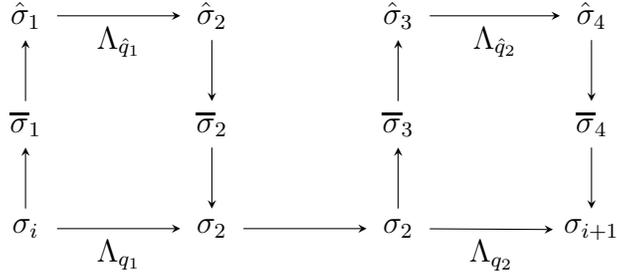
\begin{figure}[t]
\centering
\begin{tikzpicture}
  \matrix (m) [matrix of math nodes,row sep=2em,column sep=4em,minimum width=2em]
  {   \hat{\sigma}_1  & \hat{\sigma}_2  &  \hat{\sigma}_3  & \hat{\sigma}_4    \\
     \overline{\sigma}_1 &  \overline{\sigma}_2 &  \overline{\sigma}_3 &  \overline{\sigma}_4 \\
     \sigma_i  & \sigma_2 & \sigma_2  & \sigma_{i+1} \\};
  \path[-stealth]
    (m-3-1) edge node [right] {} (m-2-1)
    (m-2-1) edge node [right] {} (m-1-1)
    (m-1-2) edge node [right] {} (m-2-2)
    (m-2-2) edge node [right] {} (m-3-2)
    (m-3-1) edge node [below] {$\Lambda_{q_1}$} (m-3-2)
    (m-1-1) edge node [below] {$\Lambda_{\hat{q}_1}$} (m-1-2)
    (m-3-2) edge node[below] {} (m-3-3)
    (m-3-3) edge node [right] {} (m-2-3)
    (m-2-3) edge node [right] {} (m-1-3)
    (m-1-4) edge node [right] {} (m-2-4)
    (m-2-4) edge node [right] {} (m-3-4)
    (m-3-3) edge node [below] {$\Lambda_{q_2}$} (m-3-4)
    (m-1-3) edge node [below] {$\Lambda_{\hat{q}_2}$} (m-1-4);
\end{tikzpicture}
\caption{Chain-of-thought reasoning.}
\label{fig:cot-reasoning}
\end{figure}

To formalize CoT reasoning, we use the \textsc{Hex} framework. Consider a problem $(q, \sigma_i)$, where $q$ is a query and $\sigma_i$ is a concrete input state. We can split this problem into two sub problems $(q_1, \sigma_1)$ and $(q_2, \sigma_2)$, such that the query transformation $\Lambda_q$ is the same as applying $\Lambda_{q_1}$ first and then $\Lambda_{q_2}$, with $\sigma_i = \sigma_1$ and $\sigma_2 = \Lambda_{q_1}(\sigma_1)$, other words, $\Lambda_q(\sigma_i) = (\Lambda_{q_2} \circ \Lambda_{q_1}) (\sigma_i)$.

Figure~\ref{fig:cot-reasoning} illustrates the \textsc{Hex} diagram for CoT reasoning, which combines the \textsc{Hex} diagrams for two sub-problems: $(q_1, \sigma_1)$ and $(q_2, \sigma_2)$,
where $\sigma_i = \sigma_1$.

\begin{lemma}
\label{lem:cot-reasoning}
Suppose we have a problem $(q, \sigma_i)$ that can be solved by splitting it into two sub-problems $(q_1, \sigma_1)$ and $(q_2, \sigma_2)$, such that $\Lambda(\sigma_i) = (\Lambda_{q_2} \circ \Lambda_{q_1}) (\sigma_i)$. 
Suppose also that Figure~\ref{fig:cot-reasoning} shows how to use the CoT reasoning method to solve $(q, \sigma_i)$. Then the CoT reasoning is valid if and only if every possible way of transforming $\sigma_i$ into $\Lambda_q(\sigma_i)$ in the diagram gives the same result, i.e., the diagram commutes.
\end{lemma}
\begin{proof}
Follows from that fact that commutativity implies no hallucination or misalignment.
\end{proof}

\begin{figure}[t]
\centering
\begin{tikzpicture}
  \matrix (m) [matrix of math nodes,row sep=2em,column sep=4em,minimum width=2em]
  {  \hat{\sigma}_1  & & \hat{\sigma}_2    \\
     \overline{\sigma}_1 &  &  \overline{\sigma}_2 \\
     \sigma_i  & \sigma_2  & \sigma_{i+1} \\};
  \path[-stealth]
    (m-3-1) edge node [right] {} (m-2-1)
    (m-2-1) edge node [right] {} (m-1-1)
    (m-3-1) edge node [below] {$\Lambda_{q_1}$} (m-3-2)
    (m-1-1) edge node [below] {$\Lambda_{\hat{q_1 . q_2}}$} (m-1-3)
    (m-2-3) edge node [right] {} (m-3-3)
    (m-1-3) edge node [right] {} (m-2-3)
    (m-3-2) edge node [below] {$\Lambda_{q_2}$} (m-3-3);
\end{tikzpicture}
\caption{Chain-of-thought prompting.}
\label{fig:cot-prompting}
\end{figure}
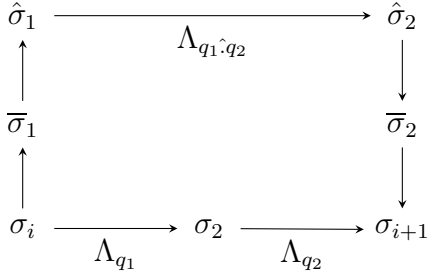

To elicit CoT reasoning, \textit{CoT prompting} embeds it in the LLM's prompt~\cite{wei2023chainofthought}. This method is convenient because it avoids extra calls to the LLM, which can be costly. Figure~\ref{fig:cot-prompting} illustrates the \textsc{Hex} diagram for CoT prompting -- in particular,  we split the problem $(q, \sigma_i)$ into two sub problems $(q_1, \sigma_1)$ and $(q_2, \sigma_2)$, such that 
$\Lambda_q = \Lambda_{q_1} \circ \Lambda_{q_2}$ -- in other words, CoT reasoning holds for the concrete problem. 
Let  $(\overline{q}_1, \_)$ and $(\overline{q}_2, \_)$ be the prompts for the problems $(q_1, \sigma_1)$ and $(q_2, \sigma_2)$ respectively. Then we concatenate the query strings of the subproblem prompts to form the CoT prompt for the original problem: $(\overline{q}_1 . \overline{q_2}, \sigma_i)$. We denote the abstract query corresponding to the CoT prompt by $\hat{q_1 . q_2}$, and the LLM function that takes the CoT prompt as input and produces the text state $\overline{\sigma}_2$ as the output by $\Lambda_{\hat{q_1 . q_2}}$. Then the following holds.

\newpage

\begin{lemma}
\label{lem:cot-prompting}
Suppose we have a problem $(q, \sigma_i)$ that can be solved by splitting it into two sub-problems $(q_1, \sigma_1)$ and $(q_2, \sigma_2)$, such that $\Lambda(\sigma_i) = (\Lambda_{q_2} \circ \Lambda_{q_1}) (\sigma_i)$. 
Suppose also that Figure~\ref{fig:cot-prompting} shows how to use the CoT prompting method to solve $(q, \sigma_i)$. Then CoT prompting is valid if and only if every possible way of transforming $\sigma_i$ into $\Lambda_q(\sigma_i)$ in the diagram gives the same result, i.e., the diagram commutes.
 \end{lemma}
 \begin{proof}
Follows from that fact that commutativity implies no hallucination or misalignment.
 \end{proof}

\begin{corollary}
\label{cor:cot}
CoT prompting and CoT reasoning produce the same output for a problem $(q, \sigma)$ if and only if both of these conditions are true:
\begin{itemize}[noitemsep, topsep=0pt]
    \item CoT reasoning holds for the problem $(q, \sigma)$. That is, $\Lambda_q = \Lambda_{q_1} \circ \Lambda_{q_2}$.
    \item The \textsc{Hex} diagrams for CoT prompting and CoT reasoning are both commutative.
\end{itemize}
\end{corollary}
CoT prompting and CoT reasoning are equivalent only under the conditions given by Corollary~\ref{cor:cot}. However, many systems (including engineering APIs) rely on CoT prompting without verifying this. This can lead to incorrect outcomes, as we will illustrate. In general, \textit{CoT prompting is not a sufficient for enabling CoT reasoning!}

The following example shows a zero-shot prompting task for GPT-4: generate TikZ code to draw a circle inside a hexagon, and a square inside the circle. However, GPT-4's solution is incorrect. It draws a circle inside a hexagon and a square inside the circle, as shown below.

\begin{mytextbox}
\textbf{Zero-shot prompting with GPT-4} \\

\textbf{Problem}: Draw a circle inside an hexagon and a square inside the circle using TikZ. \\

\textbf{Prompt:}

\begin{verbatim}
final_answer = GPT4("You are a programmer that uses TikZ to generate diagrams. 
   Draw a circle inside an hexagon and a square inside the circle.")
\end{verbatim}

\textbf{Result after compiling the generated code:}
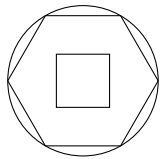
\begin{center}
\begin{tikzpicture}
  \draw (0:1) \foreach \i in {60,120,...,360} {  -- (\i:1) };
  \draw (0,0) circle (1);
  \pgfmathsetmacro{\squareSide}{1/sqrt(2)}
  \draw (-\squareSide/2,-\squareSide/2) rectangle (\squareSide/2,\squareSide/2);
\end{tikzpicture}
\end{center}
\end{mytextbox}

 \noindent
With CoT prompting, or trying to solve the problem step by step in the same prompt, GPT-4 draws a square inside a circle and the square inside another circle, as shown below. Unfortunately, this is still incorrect and in fact, worse than the zero-shot instance, which did not draw an incorrect shape (the square instead of a hexagon).

\begin{mytextbox}
\textbf{CoT prompting with GPT-4}\\

\textbf{Prompt:} 

\begin{verbatim}
final_answer = GPT4("You are a programmer that uses TikZ to generate diagrams. 
   First, draw a hexagon. Then, draw a circle inside the hexagon. Then, draw 
   a square inside the circle.")
\end{verbatim}

\textbf{Result after compiling the generated code:}
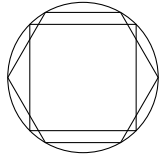
\begin{center}

\begin{tikzpicture}
\draw (0:1cm) -- (60:1cm) -- (120:1cm) -- (180:1cm) -- (240:1cm) -- (300:1cm) -- cycle; 

\draw (0,0) circle (1cm);

\draw (-0.707cm,-0.707cm) rectangle (0.707cm,0.707cm);
    
\end{tikzpicture}

\end{center}




\end{mytextbox}

\noindent
Finally, with CoT reasoning, GPT-4 draws the correct figure.
\begin{mytextbox}
\textbf{CoT reasoning with GPT-4:}
\texttt{}\\

\textbf{Prompt1:} 

\begin{verbatim}
answer1 = GPT4("You are a programmer that uses TikZ to generate diagrams.
   Draw an hexagon.")
\end{verbatim}

\textbf{Prompt2:} 

\begin{verbatim}
answer2 =GPT4("You are a programmer that uses TikZ to generate diagrams. 
   A hexagon can be drawn as {answer1}. Draw a circle inside the hexagon.")
\end{verbatim}

\textbf{Prompt3} 
\begin{verbatim}
final_answer =GPT4("You are a programmer that uses TikZ to generate 
   diagrams. An hexagon can be drawn as {answer1}. A circle inside the 
   hexagon can be drawn as {answer2}. Draw a square inside the circle.")
\end{verbatim}

\textbf{Result after compiling the generated code:}
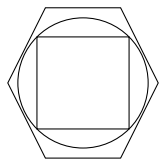
\begin{center}
\begin{tikzpicture}
\draw (0,0) -- (1,0) -- (1.5,1) -- (1,2) -- (0,2) -- (-0.5,1) -- cycle;

\coordinate (center) at (0.5, 1);

\pgfmathsetmacro{\radius}{sqrt(3)/2};

\draw (center) circle (\radius);

\pgfmathsetmacro{\squareSide}{\radius * sqrt(2)};

\draw (center) ++(-\squareSide/2, -\squareSide/2) rectangle ++(\squareSide, \squareSide);

\end{tikzpicture}
\end{center}

\end{mytextbox}

\noindent
Corollary~\ref{cor:cot} has important consequences on familiar concepts in the LLM space. These include self-verification and prompt engineering, which we will discuss next.

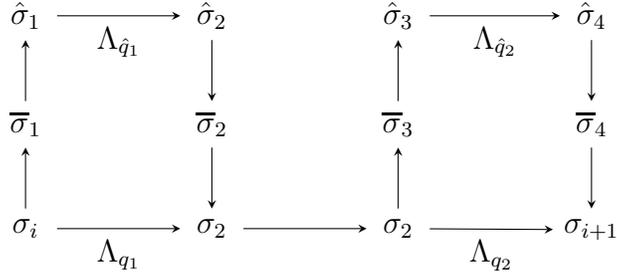
\begin{figure}[t]
\centering
\begin{tikzpicture}
  \matrix (m) [matrix of math nodes,row sep=2em,column sep=4em,minimum width=2em]
  {  \hat{\sigma}_1  & \hat{\sigma}_2  &  \hat{\sigma}_3  & \hat{\sigma}_4    \\
     \overline{\sigma}_1 &  \overline{\sigma}_2 &  \overline{\sigma}_3 &  \overline{\sigma}_4 \\
     \sigma_i  & \sigma_2 & \sigma_2  & \sigma_{i+1} \\};
  \path[-stealth]

    (m-3-1) edge node [right] {} (m-2-1)
    (m-2-1) edge node [right] {} (m-1-1)
    (m-1-2) edge node [right] {} (m-2-2)
    (m-2-2) edge node [right] {} (m-3-2)
    (m-3-1) edge node [below] {$\Lambda_{q_1}$} (m-3-2)
    (m-1-1) edge node [below] {$\Lambda_{\hat{q}_1}$} (m-1-2)
    (m-3-2) edge node[below] {} (m-3-3)
    (m-3-3) edge node [right] {} (m-2-3)
    (m-2-3) edge node [right] {} (m-1-3)
    (m-1-4) edge node [right] {} (m-2-4)
    (m-2-4) edge node [right] {} (m-3-4)
    (m-3-3) edge node [below] {$\Lambda_{q_2}$} (m-3-4)
    (m-1-3) edge node [below] {$\Lambda_{\hat{q}_2}$} (m-1-4);
\end{tikzpicture}
\caption{Self-verification.}
\label{fig:self-verification}
\end{figure}

\subsection{Self-verification}
\label{subsec:self-verification}
Self-verification~\cite{gero2023selfverification,weng2023large} is the ability of LLMs to assess the validity of their own output. In problems solved with self-verification, the same LLM plays two different roles: 1) it is a generator of solutions, and 2) a discriminator of potentially wrong answers. But how can an LLM that has produced a wrong answer check itself? We explain this using the \textsc{Hex} framework.

Self verification involves two sub-problems in one. First, a problem $(q_1, \sigma_1)$ that generates an output state $\sigma_2 = \Lambda_{q_1}(\sigma_1)$. Second, 
a verification problem $(q_2, \sigma_2)$ that takes $\sigma_2$ as input and produces a state $\sigma_{3}$ that indicates whether $\sigma_2$ is correct or not. This process can be represented by composing the \textsc{Hex} diagrams for the two sub-problems $(q_1, \sigma_1)$ and $(q_2, \sigma_2)$. This is analogous to the CoT reasoning. Through \textsc{Hex}, there is a theoretical justification for the empirical evidence in the literature that self-verification improves accuracy: a commutative diagram, like the one in Figure~\ref{fig:self-verification}, where the two sub-diagrams are the generation and verification steps.

\begin{mytextbox}
\textbf{Self-verification with generation and verification in the same prompt}
\texttt{}\\

\textbf{Prompt:} 

\begin{verbatim}
answer = GPT4("You are an expert in political history. Give me a list of 
   names of 20 politicians that were born in the city of Chicago. After 
   generating the list of names, create another list with the cities where they 
   were born.")
\end{verbatim}

\textbf{Answer:} 
\begin{verbatim}
List of Politicians born in Chicago:

1. Hillary Clinton 2. Nancy Pelosi 3. Carol Moseley Braun 4. Rahm Emanuel 
5. Dan Rostenkowski 6. Harold Washington 7. Bobby Rush 8. Rod Blagojevich
9. Timuel Black 10. John W. E. Thomas

Their respective birth cities (all in Chicago): 

1. Chicago, Illinois 2. Chicago, Illinois 3. Chicago, Illinois 4. Chicago, 
Illinois 5. Chicago, Illinois 6. Chicago, Illinois 7. Chicago, Illinois 
8. Chicago, Illinois 9. Chicago, Illinois 10. Chicago, Illinois
\end{verbatim}
\end{mytextbox}

The simple example above illustrates this. We combine the generation and verification steps in the same prompt. In the generation step, the model produces a list of 10 names, but it has no way of asserting if they were all born in Chicago. In fact, Nancy Pelosi, Bobby Rush and Timuel Black were born elsewhere. However, when we split the task into two steps -- generating the list and verifying it -- we see that the model can identify its own mistakes.

On the other hand, as seen below, splitting the problem into several LLM calls, rather than one unified prompt, helps in this example too. This is just an illustration, and the self-verification literature actually uses separate prompts for generating and verifying answers.
\begin{mytextbox}
\textbf{Self-verification with verification in a separate prompt}

\textbf{Prompt:} 
\begin{verbatim}
answer = GPT4("You are an expert in political history. Give me a list with 
   the cities where these politicians were born: 1. Hillary Clinton 2. Nancy 
   Pelosi 3. Carol Moseley Braun 4. Rahm Emanuel 5. Dan Rostenkowski 
   6. Harold Washington 7. Bobby Rush 8. Rod Blagojevich 9. Timuel Black 
   10. John W. E. Thomas.")
\end{verbatim}

\textbf{Answer:} 
\begin{verbatim}
1. Hillary Clinton - Born in Chicago, Illinois, United States
2. Nancy Pelosi - Born in Baltimore, Maryland, United States
3. Carol Moseley Braun - Born in Chicago, Illinois, United States
4. Rahm Emanuel - Born in Chicago, Illinois, United States
5. Dan Rostenkowski - Born in Chicago, Illinois, United States
6. Harold Washington - Born in Chicago, Illinois, United States
7. Bobby Rush - Born in Albany, Georgia, United States
8. Rod Blagojevich - Born in Chicago, Illinois, United States
9. Timuel Black - Born in Birmingham, Alabama, United States
10. John W. E. Thomas - Born in Springfield, Illinois, United States.
\end{verbatim}
\end{mytextbox}

Generating and verifying answers may involve multiple sub-problems, each of which could benefit from CoT reasoning. For instance, \cite{weng2023large} used CoT prompting to generate questions for verification, but we argue that CoT reasoning could also help to decompose the generation and verification processes further, when they require multiple steps. Some empirical evidence supports this idea. The recent \textit{chain-of-verification} work \cite{dhuliawala2023chainofverification} compared the performance of LLMs when they used the same prompt for both planning and executing the verification, or when they used separate prompts for each step. The results showed that the latter approach, which follows CoT reasoning principles, outperformed the former, which follows CoT prompting principles.

\subsection{Prompt programming}
\label{subsec:prompt-programming}

The text structure of a prompt can affect how large language models (LLMs) process and respond to it. We need a method to create and evaluate prompts that match our goals and expectations for LLMs, and that suit their text-processing capabilities. This is what prompt programming is about: a systematic way to design and test prompts for LLMs \cite{Reynolds2021}.

A clear and concise prompt has two basic components: a query and a state. The query tells the model what the user wants to do, such as summarize a text or caption an image. The state provides the model with the data or information it needs to do it, such as the text or the image. This separation makes it easier to design prompts that can adapt to different scenarios.

The OpenAI API in Azure~\cite{azureOpenAI} (and other platforms where GPT-4 is available) provides a natural way to separate the query and the state in the prompt. The API lets users structure the prompt into three sections: \texttt{"system"}, \texttt{"user"}, and \texttt{"assistant"}. These sections correspond to the key elements of the \textsc{Hex} diagram in Figure~\ref{fig:hex}:
\begin{itemize}
\item The \texttt{"system"} section defines the role of the LLM. It is the query component of the prompt, $\bar{q}$, that expresses the question $q$ that we want to answer.
\item The \texttt{"user"} section provides the specifics of the question. It is the state component of the prompt, $\bar{\sigma}_i$, that represents the textual version of the initial state $\sigma_i$.
\item The \texttt{"assistant"} section allows the user to write the state $\bar{\sigma}_{i+1}$ explicitly. This is useful for few-shot learning approaches, where multiple pairs of prompts and responses are given to the LLMs before the final question is asked.
\end{itemize}

The following is an example of using the \texttt{Python} API to write a zero-shot prompt for adding two numbers. This is the same example as in Section~\ref{subsec:examples}.

\begin{mytextbox}
\textbf{Zero-shot learning}

\begin{verbatim}
import openai
...

response = openai.ChatCompletion.create(
    engine="gpt-4",
    messages=[
        {"role": "system", "content": "you are calculator that can compute 
          the sum of x and y."},
        {"role": "user", "content": " x=12 and y =13."}])

LLM output: "The sum of x and y is 12 + 13 = 25."
\end{verbatim}
\end{mytextbox}
\noindent
The same API also supports few-shot learning, which requires some examples of questions and answers that are related to the query as shown below.
\begin{mytextbox}
\textbf{Few-shot learning}
\begin{verbatim}
import openai
...
response = openai.ChatCompletion.create(
    engine="gpt-4",
    messages=[
        {"role": "system", "content": "Assistant is an intelligent chatbot 
          designed to help users answer history and political questions. "},
        {"role": "user", "content": "Who was the first president of the 
          United States?"},
        {"role": "assistant", "content": "George Washington."},
        {"role": "user", "content": "What is the capital of France?"},
        {"role": "assistant", "content": "The capital of France is Paris."},
        {"role": "user", "content": "Who wrote the novel Pride and  Prejudice?"},
        {"role": "assistant", "content": "Jane Austen."},
        {"role": "user", "content": "Who is the current Prime Minister of
           the United Kingdom?"}])

LLM output: "The current Prime Minister of the United  Kingdom is 
Rishi Sunak."
\end{verbatim}
\end{mytextbox}
\noindent
This API lets us create CoT prompts by providing the assistant with step-by-step solutions of the examples as its input.
\begin{mytextbox}
\textbf{Chain-of-thought prompting\footnote{Example 
adapted from \text{https://learnprompting.org/docs/intermediate/chain\_of\_thought}}}
\begin{verbatim}
import openai
...
response = openai.ChatCompletion.create(
    engine="gpt-4",
    messages=
        {"role": "system", "content": "Assistant is a travel planner "},
        {"role": "user", "content": "Which is a faster way to get home? 
          Option 1: Take a 10 minutes bus, then a 40 minute bus, and 
          finally a 10 minute train. Option 2: Take a 90 minutes train, 
          then a 45 minute bike ride, and finally a 10 minute bus."},
        {"role": "assistant", "content": "Option 1 will take 10+40+10 = 60 
          minutes. Option 2 will take 90+45+10=145 minutes. Since 
          Option 1 takes 60 minutes and Option 2 takes 145 minutes, 
          Option 1 is faster."},
        {"role": "user", "content": "Which is a faster way to get to work? 
          Option 1: Take a 1000 minute bus, then a half hour train, and
          finally a 10 minute bike ride. Option 2: Take an 800 minute bus, 
          then an hour train, and finally a 30 minute bike ride."}])
\end{verbatim}
\end{mytextbox}

\begin{mytextbox}
\begin{verbatim}
LLM output: "Option 1 will take 1000+30+10 = 1040 minutes. Option 2 
will take 800+60+30 = 890 minutes. Since Option 2 takes 890 minutes 
and Option 1 takes 1040 minutes, Option 2 is faster."
\end{verbatim}
\end{mytextbox}

\noindent
This API gives the user a lot of flexibility to structure the problem and assign different roles to the parts of the prompt. However, it does not show how the service combines these components internally. As Section \ref{fig:cot-reasoning} explains, CoT reasoning involves strong assumptions and multiple calls to the LLM (one for each step), unlike CoT prompting that uses a single prompt. This enables more diverse and effective ways to use and deploy models like GPT-4 with CoT reasoning ideas. This is especially helpful in domains like healthcare, where a domain expert needs to define and execute sub-tasks of a problem. We hope that this work contributes to advancing these directions.

\section{Discussion}\label{sec:discussion}
LLMs are widely used in AI for various tasks, applications and services, but they also pose challenges and risks, such as generating inaccurate or misleading outputs, or failing to align with human values and expectations. We presented a mathematical framework to analyse, compare and improve LLMs in a systematic way. Our framework defines hallucination and alignment, and reveals the factors that affect our interactions with LLMs. We show that a common method of prompting LLMs, which splits a query into sub-queries, does not ensure sound reasoning. This finding has important implications for building reliable and robust AI systems with LLMs. The key idea is how to guarantee the self-consistency of the model, or how to improve both models and prompts to ensure factually correct answers that lead to commutative diagrams. Depending on the problem, this may require more detailed instructions in a CoT prompting fashion, or multiple calls to the model. We do not provide a recipe for each case, but we characterize these issues and highlight the differences. 

Our mathematical framework reveals some promising directions for further research on LLMs. A key challenge in AI is to ensure that LLMs learn abstract representations from data that respect some invariances or general principles. On that note, the \textsc{Hex} diagram shows how its components relate to various design principles or choices for LLMs. For example, prompt engineering aims to find good abstraction mappings $\alpha_c$ that make the model's input and output more consistent. On the other hand, $\gamma_c$ determines how the model produces answers. In the current conversational use of LLMs, this involves learning from human feedback (RLHF). We can also fine-tune the representations for the internal state $\hat{\sigma}_i$ and the mapping $\Lambda_{\hat{q}}$ for specific applications. We hope that our work helps address these challenges and accelerate progress in AI.

\subsection*{Acknowledgements}
We thank Javier Alvaraz-Valle, Shruthi Bannur and Ted Meeds for their valuable feedback, which helped us improve our work. We also thank Peter Lee, Sriram Rajamani and Desney Tan for inspiring conversations that sparked some of our ideas.

\vskip 0.2in
\bibliographystyle{alpha}

\newcommand{\etalchar}[1]{$^{#1}$}

\end{document}